\documentclass{article}

\usepackage{microtype}
\usepackage{graphicx}
\usepackage{subfigure}
\usepackage{booktabs} %

\usepackage{hyperref}

\usepackage[accepted]{icml2019}

\usepackage{caption}
\usepackage[utf8]{inputenc}
\usepackage{amsthm,amssymb,url,amsmath}
\usepackage{algorithm, algorithmic}
\usepackage{color}
\usepackage{todonotes}
\usepackage{balance}
\usepackage[british]{babel}
\usepackage{csquotes}
\usepackage{cleveref}
\usepackage{pgfplots}
\usepackage[inline]{enumitem}
\usepackage{thmtools}
\usepackage{mathtools}

\DeclarePairedDelimiter{\ceil}{\lceil}{\rceil}

\theoremstyle{plain}
\newtheorem{theorem}{Theorem}[section]
\newtheorem{lemma}[theorem]{Lemma}

\theoremstyle{definition}
\newtheorem{definition}[theorem]{Definition}
\newtheorem{notation}[theorem]{Notation}

\theoremstyle{remark}
\newtheorem*{remark}{Remark}

\icmltitlerunning{On the Limitations of Representing Functions on Sets}

\begin{document}

\twocolumn[
\icmltitle{On the Limitations of Representing Functions on Sets}

\icmlsetsymbol{equal}{*}

\begin{icmlauthorlist}
\icmlauthor{Edward Wagstaff}{equal,ox}
\icmlauthor{Fabian B. Fuchs}{equal,ox}
\icmlauthor{Martin Engelcke}{equal,ox}
\icmlauthor{Ingmar Posner}{ox}
\icmlauthor{Michael Osborne}{ox}
\end{icmlauthorlist}

\icmlaffiliation{ox}{Department of Engineering Science, University of Oxford, Oxford, United Kingdom}
\icmlcorrespondingauthor{}{\{ed, fabian, martin\}@robots.ox.ac.uk}

\icmlkeywords{Machine Learning, ICML}

\vskip 0.3in
]

\printAffiliationsAndNotice{\icmlEqualContribution} %

\begin{abstract}

Recent work on the representation of functions on sets has considered the use of summation in a latent space to enforce permutation invariance. In particular, it has been conjectured that the dimension of this latent space may remain fixed as the cardinality of the sets under consideration increases. However, we demonstrate that the analysis leading to this conjecture requires mappings which are highly discontinuous and argue that this is only of limited practical use. Motivated by this observation, we prove that an implementation of this model via continuous mappings (as provided by e.g. neural networks or Gaussian processes) actually imposes a constraint on the dimensionality of the latent space. Practical universal function representation for set inputs can only be achieved with a latent dimension at least the size of the maximum number of input elements.

\end{abstract}
\section{Introduction}

Machine learning models have had great success in taking advantage of structure in their input spaces:
recurrent neural networks are popular models for sequential data \citep{sutskever2014sequence} and convolutional neural networks are the state-of-the-art for many image-based problems \citep{he2015deep}.
Recently, however, models for unstructured inputs in the form of sets have rapidly gained attention \citep{ravanbakhsh2016deep, Zaheer2017, Qi2017, Lee2018, Murphy2018, Korshunova2018}.

Importantly, a range of machine learning problems can naturally be formulated in terms of sets; e.g. parsing a scene composed of a set of objects \citep{Eslami2016,Kosiorek2018}, making predictions from a set of points forming a 3D point cloud \citep{Qi2017,Qi2017a}, or training a set of agents in reinforcement learning \citep{sunehag2017value}.
Furthermore, attention-based models perform a weighted summation of a set of features \citep{Vaswani2017, Lee2018}.
Hence, understanding the mathematical properties of set-based models is valuable both in terms of set-structured applications as well as better understanding the capabilities and limitations of attention-based models.

Many popular machine learning models, including neural networks and Gaussian processes, are fundamentally based on vector inputs\footnote{Or inputs of higher rank, i.e. matrices and tensors.} rather than set inputs. In order to adapt these models for use with sets, we must enforce the property of \emph{permutation invariance}, i.e. the output of the model must not change if the inputs are reordered. Multiple authors, including \citet{ravanbakhsh2016deep}, \citet{Zaheer2017} and \citet{Qi2017}, have considered enforcing this property using a technique which we term \emph{sum-decomposition}, illustrated in \Cref{fig:figure1}. Mathematically speaking, we say that a function $f$ defined on sets of size $M$ is \emph{sum-decomposable via $Z$} if there are functions $\phi: \mathbb{R} \to Z$ and $\rho: Z \to \mathbb{R}$ such that\footnote{We use $\mathbb{R}$ here for brevity -- see \Cref{def:sum-decomp} for the fully general definition.}

\vspace{-8mm}
\begin{equation}
\label{eq:main}
f(X) = \rho \big( \Sigma_{x \in X} \phi(x) \big)
\end{equation}
\vspace{-10mm}

We refer to $Z$ here as the \emph{latent space}. Since summation is permutation-invariant, a sum-decomposition is also permutation-invariant. \citet{ravanbakhsh2016deep}, \citet{Zaheer2017} and \citet{Qi2017a} have also considered the idea of enforcing permutation invariance using other operations, e.g. $\text{\tt{max}}(\cdot)$. In this paper we concentrate on a detailed analysis of sum-decomposition, but some of the limitations we discuss also apply when $\text{\tt{max}}(\cdot)$ is used instead of summation.

\begin{figure}%
    \centering
    \includegraphics[width=0.45\textwidth]{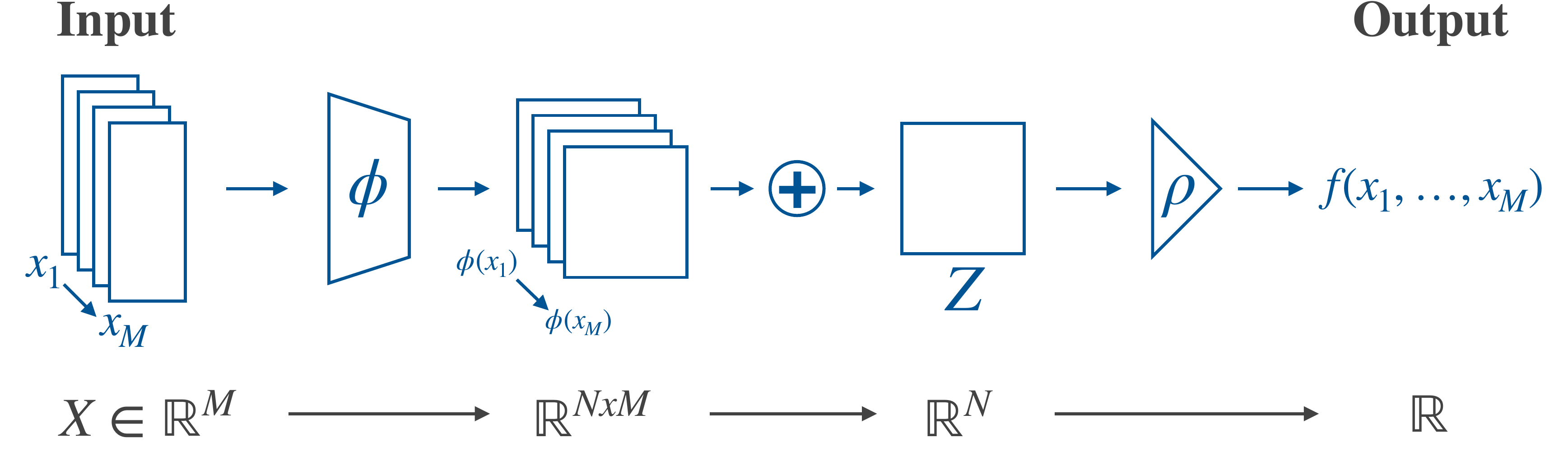}
    \caption{Illustration of the model structure proposed in several works \citep{Zaheer2017,Qi2017} for representing permutation-invariant functions.
    The sum operation enforces permutation invariance for the model as a whole. $\phi$ and $\rho$ can be implemented by e.g. neural networks.}
    \label{fig:figure1}
    \vspace{-5mm}
\end{figure}

\pagebreak

Our main contributions can be summarised as follows.
\begin{enumerate}
    \item 
    Recent proofs, e.g. in \citet{Zaheer2017}, consider functions on countable domains. 
    We explain why considering countable domains can lead to results of limited practical value (i.e. cannot be implemented with a neural network), and why considering continuity on uncountable domains such as $\mathbb{R}$ is necessary. 
    With reference to neural networks, we ground this discussion in the universal approximation theorem, which relies on continuity on uncountable domains $[0,1]^M$.
    \item In contrast to previous work \citep{Zaheer2017,Qi2017}, which considers sufficient conditions for universal function representation, we establish a \emph{necessary} condition for a sum-decomposition-based model to be capable of universal function representation. Additionally, we provide weaker sufficient conditions which imply a stronger version of universality. %
    Specifically, we show that the dimension of the latent space being at least as large as the maximum number of input elements is both necessary and sufficient for universal function representation.
\end{enumerate}

While primarily targeted at neural networks, these results hold for any implementation of sum-decomposition, e.g. using Gaussian processes, as long as it provides universal function approximation for continuous functions.
Proofs of all novel results are available in \Cref{sec:proofs}.

\section{Preliminaries}
\label{sec:preliminaries}

In this section we recount the theorems and proofs on sum-decomposition from \citet{Zaheer2017}.
We begin by introducing important definitions and the notation used throughout our work.
Note that we focus on permutation-invariant functions and do not discuss permutation equivariance which is also considered in \citet{Zaheer2017}.

\subsection{Definitions}

\begin{definition}
A function $f(\mathbf{x})$ is \emph{permutation-invariant} if $f(x_1,\dots,x_M) = f \bigl(x_{\pi(1)},\dots,x_{\pi(M)} \bigr)$ for all $\pi$.
\end{definition}

\begin{definition}
\label{def:sum-decomp}
We say that a function $f$ is \emph{sum-decomposable} if there are functions $\rho$ and $\phi$ such that

\begin{displaymath}
f(X) = \rho \bigl( \Sigma_{x \in X} \phi(x) \bigr).
\end{displaymath}

In this case, we say that $(\rho, \phi)$ is a \emph{sum-decomposition} of $f$.

Given a latent space $Z$, we say that $f$ is \emph{sum-decomposable via $Z$} when this expression holds for some $\phi$ whose codomain is $Z$, i.e. $\phi: \mathfrak{X} \to Z$.

We say that $f$ is \emph{continuously sum-decomposable} when this expression holds for some continuous functions $\rho$ and $\phi$.

We will also consider sum-decomposability where the inputs to $f$ are vectors rather than sets - in this context, the sum is over the elements of the input vector.
\end{definition}

\begin{definition}
A set $\mathfrak{X}$ is \emph{countable} if its number of elements, i.e. the cardinality, is smaller or equal to the number of elements in $\mathbb{N}$.
This includes both finite and countably infinite sets; e.g. $\mathbb{N}$, $\mathbb{Q}$, and subsets thereof.
\end{definition}

\begin{definition}
A set $\mathfrak{X}$ is \emph{uncountable} if its number of elements is greater than the number of elements in $\mathbb{N}$, e.g. $\mathbb{R}$ and certain subsets thereof.
\end{definition}

\begin{notation}
Denote the power set of a set $\mathfrak{X}$ by $2^\mathfrak{X}$.
\end{notation}

\begin{notation}
Denote the set of \emph{finite} subsets of a set $\mathfrak{X}$ by $\mathfrak{X}^{\mathcal{F}}$.
\end{notation}

\begin{notation}
Denote the set of subsets of a set $\mathfrak{X}$ containing at most $M$ elements by $\mathfrak{X}^{\leq M}$.
\end{notation}

\begin{remark}
Throughout, we discuss expressions of the form $\Phi(X) = \Sigma_{x \in X} \phi(x)$, where $X$ is a set. Note that care must be taken in interpreting this expression when $X$ is not finite -- we discuss this issue fully in \Cref{sec:infinite_sums}.
\end{remark}

\subsection{Background Theorems}
\label{sec:original_theorems}

\citet{Zaheer2017} consider the two cases where $X$ is a subset of, or drawn from, a \emph{countable} and an \emph{uncountable} universe $\mathfrak{X}$.
We now outline the theorems and proofs relating to these two cases.

\begin{theorem}[Countable case]
\label{ori_countable_theorem}
Let $f: 2^{\mathfrak{X}} \to \mathbb{R}$ where $\mathfrak{X}$ is countable.
Then $f$ is permutation-invariant if and only if it is sum-decomposable via $\mathbb{R}$.
\end{theorem}

\begin{proof}
Since $\mathfrak{X}$ is countable, each $x \in \mathfrak{X}$ can be mapped to a unique element in $\mathbb{N}$ by a function $c(x): \mathfrak{X} \to \mathbb{N}$.
Let $\Phi(X)=\sum_{x\in X} \phi(x)$. If we can choose $\phi$ so that $\Phi$ is injective, then we can set $\rho = f \circ \Phi^{-1}$, giving

\vspace{-5mm}
\begin{gather*}
f = \rho \circ \Phi \\
f(X) = \rho \bigl( \Sigma_{x \in X} \phi(x) \bigr)
\end{gather*}

i.e. f is sum-decomposable via $\mathbb{R}$.

Now consider $\phi(x) = 4^{-c(x)}$. Under this mapping, each $X \subset \mathfrak{X}$ corresponds to a unique real number expressed in base 4. Therefore $\Phi$ is injective, and the conclusion follows.
\end{proof}

\begin{remark}
This construction works for any set size $M$, and even for sets of infinite size.
However, it assumes that $X$ is a set with no repeated elements, i.e. multisets are not supported.
Specifically, the construction will fail with multisets because $\Phi$ fails to be injective if its domain includes multisets.
In \Cref{sec:multisets}, we extend \Cref{ori_countable_theorem} to also support multisets, with the restriction that infinite sets are no longer supported.
\end{remark}

\begin{theorem}[Uncountable case]
\label{ori_uncountable_theorem}
Let $M \in \mathbb{N}$, and let $f: [0,1]^M \to \mathbb{R}$ be a continuous function.
Then $f$ is permutation-invariant if and only if it is continuously sum-decomposable via $\mathbb{R}^{M+1}$.
\end{theorem}

The proof by \citet{Zaheer2017} of \Cref{ori_uncountable_theorem} is more involved than for \Cref{ori_countable_theorem}.
We do not include it here in full detail, but briefly summarise below.

\begin{enumerate}
    \item Show that the mapping $\Phi: [0,1]^M \to \mathbb{R}^{M+1}$ defined by $\Phi_q(\mathbf{x}) = \sum_{m=1}^{M} (x_m)^q$ for $q = 0,\dots,M$ is injective and continuous.\footnote{In the original proof, $\Phi$ is denoted $E$.}
    \item Show that $\Phi$ has a continuous inverse.
    \item Define $\rho: \mathbb{R}^{M+1} \to \mathbb{R}$ by $\rho = f \circ \Phi^{-1}$.
    \item Define $\phi(x): \mathbb{R} \to \mathbb{R}^{M+1}$ by $\phi_q(x) = x^q$.
    \item Note that, by definition of $\rho$ and $\phi$, $(\rho, \phi)$ is a continuous sum-decomposition of $f$ via $\mathbb{R}^{M+1}$. \qed
\end{enumerate}

\begin{remark}
\citet{Zaheer2017} conjecture that any continuous permutation-invariant function $f$ on $2^{[0,1]}$, the power set of $[0,1]$, is continuously sum-decomposable. In \Cref{sec:continuity}, we show that this is not possible, and in \Cref{sec:continuous} we show that even if the domain of $f$ is restricted to $[0,1]^\mathcal{\leq F}$, the finite subsets of $[0,1]$, then $N \geq M$ is a \emph{necessary condition} for arbitrary functions $f$ to be continuously sum-decomposable. Additionally, we prove that $N=M$ is a \emph{sufficient condition} -- implying together with the above that it is not possible to do better than this.
\end{remark}

\section{The Importance of Continuity}
\label{sec:continuity}

In this section, we argue that continuity is essential to discussions of function representation, that it has been neglected in prior work on permutation-invariant functions, and that this neglect has implications for the strength and generality of existing results.

Intuitively speaking a function is continuous if, at every point in the domain, the variation of the output can be made arbitrarily small by limiting the variation in the input. Continuity is the reason that, for instance, working to machine precision usually produces sensible results. 
Truncating to machine precision alters the input to a function slightly, but continuity ensures that the change in output is also slight.

\begin{figure}[ht!]
\input{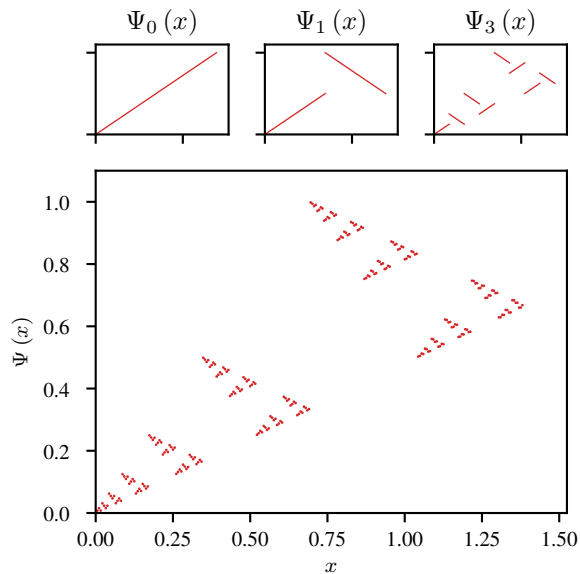}
\caption{The function $\Psi$ shown here is continuous at every rational point in $[0, \ln 4]$. Intuitively, this is because all jumps occur at irrational values, namely at certain fractions of $\ln 4$. It defies our intuitions for what continuity should mean, and illustrates the fact that continuity on $\mathbb{Q}$ is a much weaker property than continuity on $\mathbb{R}$. The latter property is required to satisfy the universal approximation theorem for neural networks. $\Psi$ is defined and discussed in \Cref{sec:function_on_q}.
}
\label{fig:q_continuous}
\end{figure}

In \citet{Zaheer2017}, the authors demonstrate that when $\mathfrak{X}$ is a countable set, e.g. the rational numbers, any function $f : 2^\mathfrak{X} \to \mathbb{R}$ is sum-decomposable via $\mathbb{R}$.
This is taken as a hopeful indication that sum-decomposability may extend to uncountable domains, e.g. $\mathfrak{X} = \mathbb{R}$.
Extending to the uncountable case may appear, at first glance, to be a mere formality -- we are, after all, ultimately interested in implementing algorithms on finite hardware.
Nevertheless, it is not true that a theoretical result for a countably infinite domain must be strong enough for practical purposes.
In fact, considering functions on uncountably infinite domains such as $\mathbb{R}^N$ is of real importance.

Turning specifically to neural networks, the universal approximation theorem says that any \emph{continuous} function can be approximated by a neural network, but not that \emph{any} function can be approximated by a neural network \cite{cybenko1989approximation}.
A similar statement is true for other approximators, such as some Gaussian processes \cite{Rasmussen2006GPML}.
The notion of continuity required here is specifically that of continuity on compact subsets of $\mathbb{R}^N$.

Crucially, if we wish to work mathematically with continuity in a way that closely matches our intuitions, we \emph{must} consider uncountable domains.
To illustrate this point, consider the rational numbers $\mathbb{Q}$. $\mathbb{Q}$ is dense in $\mathbb{R}$, and it is tempting to think that $\mathbb{Q}$ is therefore ``all we need''. 
However, a theoretical guarantee of continuity on $\mathbb{Q}$ is weak, and does not imply continuity on $\mathbb{R}$.
The universal approximation theorem for neural networks relies on continuity on $\mathbb{R}$, and we cannot usefully take continuity on $\mathbb{Q}$ as a proxy for this property.
\Cref{fig:q_continuous} shows a function which is continuous on $\mathbb{Q}$, and illustrates that a continuous function on $\mathbb{Q}$ may not extend continuously to $\mathbb{R}$.
This figure also illustrates that continuity on $\mathbb{Q}$ defies our intuitions about what continuity should mean, and is too weak for the universal approximation theorem for neural networks.
We require the stronger notion of continuity on $\mathbb{R}$.

In light of the above, it is clear that continuity is a key property for function representation, and also that there is a crucially important difference between countable and uncountable domains. This raises two problems for \Cref{ori_countable_theorem}. First, the theorem does not consider the continuity of the sum-decomposition when the domain $\mathfrak{X}$ has some non-trivial topological structure (e.g. $\mathfrak{X} = \mathbb{Q}$). Second, we still care about continuity on $\mathbb{R}$, and there is no guarantee that this is possible given continuity on $\mathbb{Q}$.

In fact, the continuity issue cannot be overcome -- we can demonstrate that in general the sum-decomposition of \Cref{ori_countable_theorem}, which goes via $\mathbb{R}$, cannot be made continuous for $\mathfrak{X}=\mathbb{Q}$:

\begin{restatable}{theorem}{qdiscontinuous}
\label{thm:q_discontinuous}
There exist functions $f : 2^\mathbb{Q} \to \mathbb{R}$ such that, whenever $(\rho, \phi)$ is a sum-decomposition of $f$ via $\mathbb{R}$, $\phi$ is discontinuous at every point $q \in \mathbb{Q}$.
\end{restatable}

We can actually say something more general than the above. Our proof can easily be adapted to demonstrate that if $f$ is injective, or if we want a fixed $\phi$ to suffice for any $f$, then $\phi$ can only be continuous at isolated points of the underlying set $\mathfrak{X}$, regardless of whether $\mathfrak{X}=\mathbb{Q}$. I.e., it is not specifically due to the structure of $\mathbb{Q}$ that continuous sum-decomposability fails. In fact, it fails whenever we have a non-trivial topological structure. For functions which we want to model using a neural network, this is worrying.

It is not possible to represent an everywhere-discontinuous $\phi$ with a neural network. We therefore view \Cref{ori_countable_theorem} as being of limited practical relevance and as not providing a reliable intuition for what should be possible in the uncountable case. We do however see this result as mathematically interesting, and have obtained the following result extending it to the case where the domain $\mathfrak{X}$ is uncountable. This result is slightly weaker than the countable case, in that the domain of $f$ can contain arbitrarily large finite sets, but not infinite sets.

\begin{restatable}{theorem}{uncdiscontinuous}
\label{thm:uncountable_discontinuous}
Let $f : \mathbb{R}^{\mathcal{F}} \to \mathbb{R}$. Then $f$ is sum-decomposable via $\mathbb{R}$.
\end{restatable}

Once again, the sum-decomposition is highly discontinuous. The limitation that $f$ is not defined on infinite sets cannot be overcome:

\begin{restatable}{theorem}{uncfinite}
\label{thm:uncountable_only_finite_subsets}
If $\mathfrak{X}$ is uncountable, then there exist functions $f : 2^\mathfrak{X} \to \mathbb{R}$ which are not sum-decomposable. Note that this holds even if the sum-decomposition $(\rho, \phi)$ is allowed to be discontinuous.
\end{restatable}

To summarise, we show why considering countable domains can lead to results of limited practical value and why considering continuity on uncountable domains is necessary.
We point out that some of the previous work is therefore of limited practical relevance, but regard it as mathematically interesting.
In this vein, we extend the analysis of sum-decomposability when continuity is not required.

\section{Practical Function Representation}
\label{sec:continuous}

Having established the necessity of considering continuity on $\mathbb{R}$, we now explore the implications for sum-decomposability of permutation-invariant functions. These considerations lead to concrete recommendations for model design and provide theoretical support for elements of current practice in the area.
Specifically, we present three theorems whose implications can be summarised as follows.

\begin{enumerate}
    \item A latent dimensionality of $M$ is \emph{sufficient} for representing all continuous permutation-invariant functions on sets of size $\leq M$.
    \item To guarantee that all continuous permutation-invariant functions can be represented for sets of size $\leq~M$, a latent dimensionality of at least $M$ is \emph{necessary}.
\end{enumerate}

The key result which is the basis of the second statement and which underpins this discussion is as follows.

\begin{restatable}{theorem}{maxnotdecomp}
\label{thm:max_not_decomposable}
Let $M > N \in \mathbb{N}$. Then there exist permutation invariant continuous functions $f : \mathbb{R}^M \to \mathbb{R}$ which are {\upshape\bfseries not} continuously sum-decomposable via $\mathbb{R}^N$.
\end{restatable}

Restated in more practical terms, this implies that for a sum-decomposition-based model to be capable of representing \emph{arbitrary} continuous functions on sets of size $M$, the latent space in which the summation happens must be chosen to have dimension at least $M$. A similar statement is true for the analogous concept of \emph{max-decomposition} -- details are available in \Cref{app:max-decomp}.

To prove this theorem, we first need to state and prove the following lemma.

\begin{lemma}
\label{lem:injection}
Let $M, N \in \mathbb{N}$, and suppose $\phi : \mathbb{R} \to \mathbb{R}^N$, $\rho : \mathbb{R}^N \to \mathbb{R}$ are functions such that:

\begin{equation}
\label{eq:max_rep}
\text{max}(X) = \rho \left( \Sigma_{x \in X} \phi(x) \right)
\end{equation}

Now let $\Phi(X) = \Sigma_{x \in X} \phi(x)$, and write $\Phi_M$ for the restriction of $\Phi$ to sets of size $M$.

Then $\Phi_M$ is injective for all $M$.
\end{lemma}

\begin{proof}
We proceed by induction. The base case $M=1$ is clear.

Now let $M \in \mathbb{N}$, and suppose that $\Phi_{M-1}$ is injective. Now suppose there are sets $X, Y$ such that $\Phi_M(X) = \Phi_M(Y)$. First note that, by \eqref{eq:max_rep}, we must have:

\begin{equation}
\label{eq:equal_max}
\text{max}(X) = \text{max}(Y)
\end{equation}

So now write:

\begin{equation}
\label{eq:decomp}
X = \{x_\text{max}\} \cup X_\text{rem} ~ ; ~ Y = \{y_\text{max}\} \cup Y_\text{rem}
\end{equation}

where $x_\text{max} = \text{max}(X)$, and similarly for $y_\text{max}$.

But now:

\begin{displaymath}
\begin{split}
\Phi_M(X) & = \Phi_{M-1}(X_\text{rem}) + \phi(x_\text{max}) \\
          & = \Phi_{M-1}(Y_\text{rem}) + \phi(y_\text{max}) \\
          & = \Phi_M(Y)
\end{split}
\end{displaymath}

From the central equality, and \eqref{eq:equal_max}, we have:

\begin{displaymath}
\Phi_{M-1}(X_\text{rem}) = \Phi_{M-1}(Y_\text{rem})
\end{displaymath}

Now by injectivity of $\Phi_{M-1}$, we have $X_\text{rem} = Y_\text{rem}$. Combining this with \eqref{eq:equal_max} and \eqref{eq:decomp}, we must have $X = Y$, and so $\Phi_M$ is injective.
\end{proof}

Equipped with this lemma, we can now prove \Cref{thm:max_not_decomposable}.

\begin{proof}

We proceed by contradiction. Suppose that functions $\phi$ and $\rho$ exist satisfying \eqref{eq:max_rep}. Define $\Phi_M : \mathbb{R}^M \to \mathbb{R}^N$ by:

\begin{displaymath}
\Phi_M(\mathbf{x}) = \Sigma_{i=1}^M \phi(x_i)
\end{displaymath}

Denote the set of all $x \in \mathbb{R}^M$ with $x_1 < x_2 < ... < x_M$ by $\mathbb{R}_\text{ord}^M$, and let $\Phi_M^\text{ord}$ be the restriction of $\Phi_M$ to $\mathbb{R}_\text{ord}^M$. Since $\Phi_M^\text{ord}$ is a sum of continuous functions, it is also continuous, and by \Cref{lem:injection}, it is injective.

Now note that $\mathbb{R}_\text{ord}^M$ is a convex open subset of $\mathbb{R}^M$, and is therefore homeomorphic to $\mathbb{R}^M$. Therefore, our continuous injective $\Phi_M^\text{ord}$ can be used to construct a continuous injection from $\mathbb{R}^M$ to $\mathbb{R}^N$. But it is well known that no such continuous injection exists when $M>N$. Therefore our decomposition \eqref{eq:max_rep} cannot exist.
\end{proof}

It is crucial to note that functions $f$ for which a lower-dimensional sum-decomposition does not exist need not be ``badly-behaved'' or difficult to specify. The limitation extends to functions of genuine interest. For our proof, we have specifically demonstrated that even $\text{\tt{max}}(X)$ is not continuously sum-decomposable when $N < M$.

From \Cref{ori_uncountable_theorem}, we also know that for a fixed input set size $M$, any continuous permutation-invariant function is continuously sum-decomposable via $\mathbb{R}^{M+1}$. It is, however, possible to adapt the construction of \citet{Zaheer2017} to strengthen the result in two ways. Firstly, we can perform the sum-decomposition via $\mathbb{R}^M$:

\begin{restatable}[Fixed set size]{theorem}{oriunc}
\label{cor:ori_uncountable_theorem}
Let $f: \mathbb{R}^{M} \to \mathbb{R}$ be continuous. Then $f$ is permutation-invariant if and only if it is continuously sum-decomposable via $\mathbb{R}^M$.
\end{restatable}

Secondly, we can deal with variable set sizes $\leq M$:

\begin{restatable}[Variable set size]{theorem}{arbitrary}
\label{thm:arbitrary_set_sizes}
    Let $f: \mathbb{R}^{\leq M} \to \mathbb{R}$ be continuous. Then $f$ is permutation-invariant if and only if it is continuously sum-decomposable via $\mathbb{R}^M$.
\end{restatable}

Note that we must take some care over the notion of continuity in this theorem -- see \Cref{sec:cont_set_fun_remark}.

\subsection{Discussion}

\Cref{thm:max_not_decomposable} does not imply \emph{all} functions require $N=M$. Some functions, such as the mean, can be represented in a lower dimensional space. The statement rather says that if we do not want to impose any limitations on the complexity of the function, the latent space needs to have dimensionality at least $M$.

\Cref{thm:arbitrary_set_sizes} suggests that sum-decomposition via a latent space with dimension $N=M$ should suffice to model any function.
Neural network models in the recent literature, however, deviate from these guidelines in several ways, indicating a disconnect between theory and practice.
For example, the models in \citet{Zaheer2017} and \citet{Qi2017} are considerably more complex than \Cref{eq:main}, e.g. they apply several permutation-equivariant layers to the input before a permutation-invariant layer.

In light of \Cref{thm:max_not_decomposable}, this disconnect becomes less surprising.
We have shown that, for a target function of sufficient complexity, $N=M$ is the bare minimum required for the model to be capable of representing the target function.
Achieving this would rely on the parameterisation of $\phi$ and $\rho$ being flexible enough and on the availability of a suitable optimisation method.
In practice, we should not be surprised that more than the bare minimum capacity in our model is required for good performance.
Even with $N>M$, the model might not converge to the desired solution.
At the same time, when we are dealing with real datasets, the training data may contain noise and redundant information, e.g. in the form of correlations between elements in the input, inducing functions of limited complexity that may in fact be representable with $N<M$.

\subsection{Illustrative Example}

We now use a toy example to illustrate some practical implications of our results.
Based on \Cref{thm:max_not_decomposable}, we expect the number of input elements $M$ to have an influence on the required latent dimension $N$, and in particular, we expect that the required latent dimension may increase without bound.

We train a neural network with the architecture presented in \Cref{fig:figure1} to predict the median of a set of values.
We choose the median as a function because it is relatively simple but cannot be trivially represented via a sum in a fixed-dimensional latent space, in contrast to e.g. the mean, which is sum-decomposable via $\mathbb{R}$.\footnote{The construction for $Z=\mathbb{R}$ is not entirely trivial for variable set size, but going via $Z=\mathbb{R}^2$ is straightforward.}
$\phi$ and $\rho$ are parameterised by multi-layer perceptrons (MLPs).
The input sets are randomly drawn from either a uniform, a Gaussian, or a Gamma distribution.

\begin{figure}[ht!]
\centering
\subfigure[Test performance on median estimation depending on latent dimension. Different colours depict different set sizes. Each data point is averaged over 500 runs with different seeds. Shaded areas indicate confidence intervals. Coloured dashed lines indicate $N=M$.]
{\label{fig:smooth_logxy}\includegraphics[width=70mm]{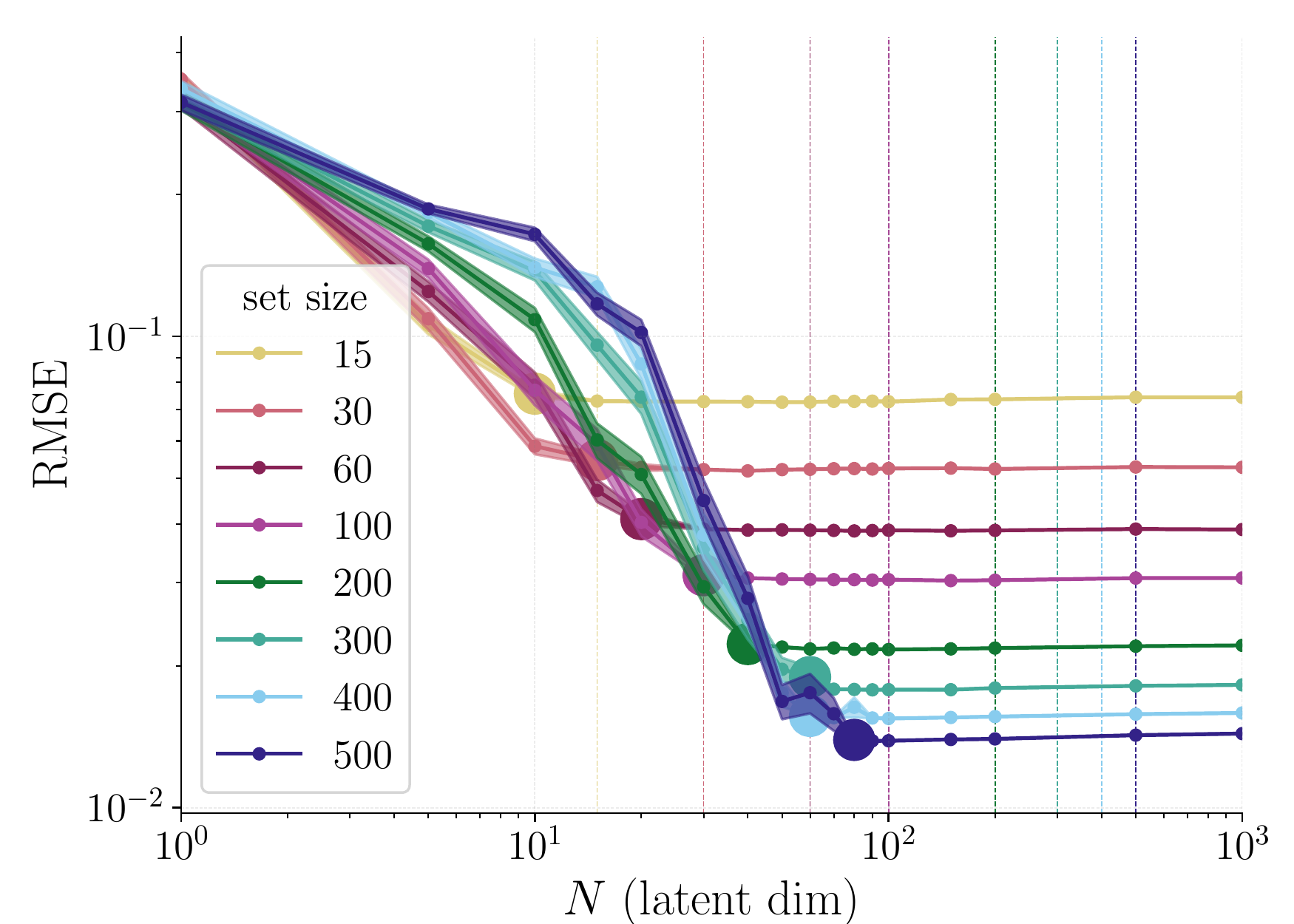}}
\subfigure[Extracted `critical points' from above graph. The coloured data points depict minimum latent dimension for optimal performance (RMSE less than $10\%$ above minimum value for this set size) for different set sizes.]{\label{fig:smooth_turning_points}\includegraphics[width=70mm]{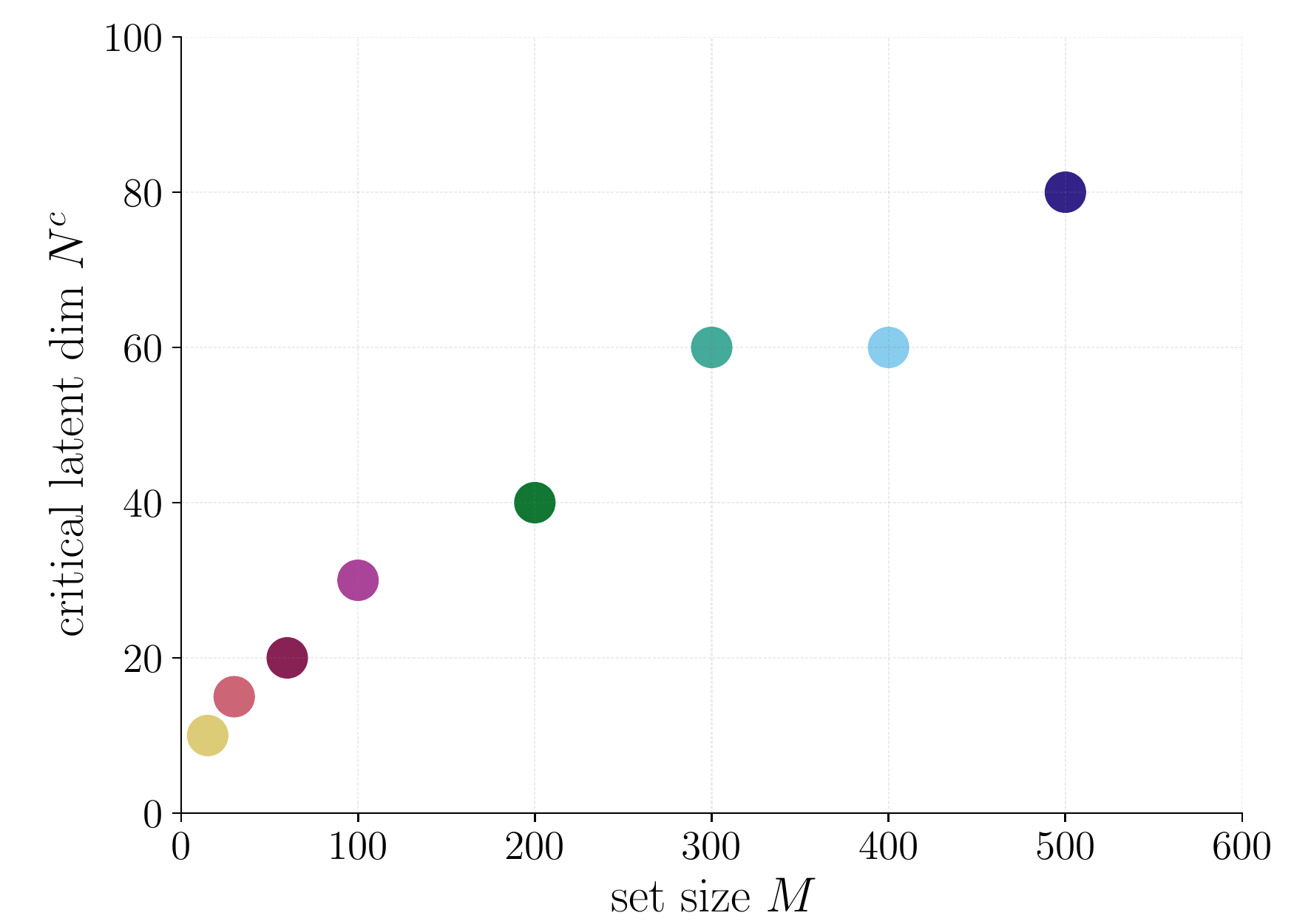}}
\vspace{-3mm}
\caption{Illustrative toy example: a neural network is trained to predict the median of an unordered set.
}
\end{figure}

We vary the latent dimension $N$ and the input set size $M$ to investigate the link between these two variables and the predictive performance.
The MLPs parameterising $\phi$ and $\rho$ are given comparatively many layers and hidden units, relative to the simplicity of the task, to ensure that the latent dimension is the bottleneck.
Further details are described in \Cref{sec:app_exp}.

\Cref{fig:smooth_logxy} shows the RMSE depending on the latent dimension for different input sizes.
We make three observations.
\begin{enumerate}
    \item For each set size, the error decreases monotonically with the dimension of the latent space.
    \item Beyond a certain point, increasing the dimension of the latent space does not further reduce the error. We denote this the ``critical point''.
    \item As the set size increases, so does the latent dimension at the critical point.
\end{enumerate}

\Cref{fig:smooth_turning_points} shows the critical points as a function of the input size, indicating a roughly linear relationship between the two.
Note that the critical points occur at $N<M$.
This can be explained by the fact that the models do not learn an algorithmic solution for computing the median, but rather to estimate it given samples drawn from the specific input distribution seen during training.
Furthermore, estimating the median of a distribution, like other functions, renders some information in the input redundant. Therefore, the mapping from input to latent space does not need to be injective, allowing a model to solve the task with a smaller value of $N$.

\section{Related Work}
\label{sec:related_work}

Much of the recent work on deep learning with unordered sets follows the paradigm discussed in \cite{ravanbakhsh2016deep}, \citet{Zaheer2017}, and \citet{Qi2017} which leverage the structure illustrated in \Cref{fig:figure1}.
\citet{Zaheer2017} provide an in-depth theoretical analysis which is discussed in detail in \Cref{sec:preliminaries}.
\citet{Qi2017} also derive a sufficiency condition for universal function approximation.
In their proof, however, they set the latent dimension $N$ to $\left\lceil 1 / \delta _ { \epsilon } \right\rceil$ where $\delta _ { \epsilon }$ depends on the error tolerance for how closely the target function has to be approximated.
As a result, the latent dimension $N$ goes to infinity for exact representation.
In similar vain, \citet{herzig2018mapping} consider permutation-invariant functions on graphs.

A key application domain of set-based methods is the processing of point clouds, as the constituent points do not have an intrinsic ordering.
The work by \citet{Qi2017} on 3D point clouds, one of the first to use a permutation-invariant neural networks, is extended in \citet{Qi2017a} by sampling and grouping points in a hierarchical fashion to model the interaction between nearby points in the input space more explicitly.
\citet{Qi2018} combine RGB and lidar data for object detection by using image detectors to generate bounding box proposals which are then further processed by a set-based model.
\citet{Achlioptas2018} and \citet{Cloud2018} show that set-based models can also be used to learn generative models of point clouds.

\citet{vinyals2015order} suggest that even though recurrent networks are universal approximators, the ordering of the input is crucial for good performance. Hence, they propose model that relies on attention to achieve permutation invariance in order to solve a sorting task.
In general, it is worth noting that there exists a connection between the model in \citet{Zaheer2017} and recent attention-based models such as the one proposed in \citet{Vaswani2017}.
In this case, the aggregation layer includes a weighting parameter which is computed based on a key-query system which is also permutation invariant.
Since the value of the weighting parameters could be learned to be $1.0$, it is trivial to show that such an attention algorithm is also in principle able to approximate any permutation-invariant function, of course depending on the remaining parts of the architecture.
Inspired by inducing point methods, Set Transformer \citep{Lee2018} propose a computationally more efficient attention-module and demonstrate better performance on a range of set-based tasks.
While stacking several of attention-modules can capture higher order dependencies, a more general treatment of this is offered by permutation-invariant, learnable Janossy Pooling \citep{Murphy2018}.

Similar to the methods considered here, Neural Processes \citep{Garnelo2018a} and Conditional Neural Processes \citep{Garnelo2018} also rely on aggregation via summation in order to infer a distribution from a set of data points.
\citet{kim2019attentive} add an attention mechanism to neural processes to improve empirical performance.
Generative Query Networks \citep{Eslami2018, Kumar2018} can be regarded as an instantiation of neural processes to learn useful representations of 3D scenes from multiple 2D views.
\citet{Yang2018} also aggregate information from multiple views to compute representations of 3D objects.

\citet{bloem2019probabilistic} and \citet{Korshunova2018} consider exchangeable sequences -- sequences consisting of random variables with a joint likelihood which is invariant under permutations.
\citet{bloem2019probabilistic} provide a theorem that describes distribution-invariant models.
\citet{Korshunova2018} use RealNVP \citep{dinh2016density} as a bijective function which sequentially computes the parameters of a Student-t process.

\section{Conclusions}
\label{sec:conclusions}

This work derives theoretical limitations on the representation of \emph{arbitrary} functions on sets via a finite latent space.
We demonstrate why continuity requires statements on uncountable domains, as opposed to countable domains, to ensure the practical usefulness of those statements.
Under this constraint, we prove that a latent space whose dimension is at least as large as the maximum input set size is both sufficient and necessary to achieve \emph{universal} function representation.
The models covered in this analysis are popular for a range of practical applications and can be implemented e.g. by neural networks or Gaussian processes.
In future work, we would like to investigate the effect of constructing models with both permutation-equivariant and permutation-invariant modules on the required dimension of the latent space.
Examining the implications of using self-attention, e.g. as in \citet{Lee2018}, would be of similar interest.

\section*{Acknowledgements}
This research was funded by the EPSRC AIMS Centre for Doctoral Training at the University of Oxford, an EPSRC DTA studentship, a Google studentship, and an EPSRC Programme Grant (EP/M019918/1).
The authors acknowledge use of Hartree Centre resources in this work. The STFC Hartree Centre is a research collaboratory in association with IBM providing High Performance Computing platforms funded by the UK’s investment in e-Infrastructure.
The authors thank Sudhanshu Kasewa and Olga Isupova for proof reading a draft of the paper.

\bibliography{references}
\bibliographystyle{icml2019}

\clearpage
\appendix
\section{Mathematical Remarks}

\subsection{Infinite Sums}
\label{sec:infinite_sums}

Throughout this paper we consider expressions of the following form:

\begin{equation}
\label{eq:setsum}
\Phi(X) = \Sigma_{x \in X} \phi(x)
\end{equation}

Where $X$ is an arbitrary set. The meaning of this expression is clear when $X$ is finite, but when $X$ is infinite, we must be precise about what we mean.

\subsubsection{Countable Sums}

We usually denote countable sums as e.g. $\Sigma_{i=1}^\infty x_i$. Note that there is an ordering of the $x_i$ here, whereas there is no ordering in our expression \eqref{eq:setsum}. The reason that we consider sums is for their permutation invariance in the finite case, but note that in the infinite case, permutation invariance of sums does not necessarily hold! For instance, the alternating harmonic series $\Sigma_{i=1}^\infty \frac{(-1)^i}{i}$ can be made to converge to any real number simply by reordering the terms of the sum. For expressions like \eqref{eq:setsum} to make sense, we must require that the sums in question are indeed permutation invariant. This property is known as \emph{absolute convergence}, and it is equivalent to the property that the sum of absolute values of the series converges. So for \eqref{eq:setsum} to make sense, we will require everywhere that $\Sigma_{x \in X} |\phi(x)|$ is convergent. For any $X$ where this is not the case, we will set $\Phi(X) = \infty$.

\subsubsection{Uncountable Sums}

It is well known that a sum over an uncountable set of elements only converges if all but countably many elements are 0. Allowing sums over uncountable sets is therefore of little interest, since it essentially reduces to the countable case.

\subsection{Continuity of Functions on Sets}
\label{sec:cont_set_fun_remark}

We are interested in functions on subsets of $\mathbb{R}$, i.e. elements of $2^\mathbb{R}$, and the notion of continuity on $2^\mathbb{R}$ is not straightforward. As a convenient shorthand, we discuss ``continuous'' functions $f$ on $2^\mathbb{R}$, but what we mean by this is that the function $f_M$ induced by $f$ on $\mathbb{R}^M$ by $f_N(x_1, ..., x_M) = f(\{x_1, ..., x_M\})$ is continuous for every $M \in \mathbb{N}$.

\subsection{Remark on \Cref{ori_countable_theorem}}
\label{sec:multisets}

The proof for \Cref{ori_countable_theorem} from \citet{Zaheer2017} can be extended to dealing with multi sets, i.e. sets with repeated elements. To that end, we replace the mapping to natural numbers $c(X): \mathbb{R}^M \to \mathbb{N}$ with a mapping to prime numbers $p(X): \mathbb{R}^M \to \mathbb{P}$. We then choose $\phi(x_m) = -\log p(x_m)$. Therefore,
\begin{equation}
\Phi(X) = \sum_{m=1}^{M} \phi (x_m) = \log \prod_{m=1}^{M} \frac{1}{p(x_m)}
\end{equation}
which takes a unique value for each distinct $X$ therefore extending the validity of the proof to multi-sets. However, unlike the original series, this choice of $\phi$ diverges with infinite set size.

In fact, it is straightforward to show that there is no function $\phi$ for which $\Phi$ provides a unique mapping for arbitrary multi-sets while at same time guaranteeing convergence for infinitely large sets. Assume a function $\phi$ and an arbitrary point $x$ such that $\phi(x) = a \neq 0$. Then, the multiset comprising infinitely many identical members $x$ would give: 
\begin{equation}
\Phi(X) = \sum_{i=1}^{\infty} \phi (x_m) = \sum_{i=1}^{\infty} a = \pm\infty
\end{equation}

\section{Proofs of Theorems}
\label{sec:proofs}

\subsection{\Cref{thm:q_discontinuous}}

\qdiscontinuous*

\begin{proof}
Consider $f(X) = \text{sup}(X)$, the least upper bound of $X$. Write $\Phi(X) = \Sigma_{x \in X} \phi(x)$. So we have:

\begin{displaymath}
\text{sup}(X) = \rho(\Phi(X))
\end{displaymath}

First note that $\phi(q) \neq 0$ for any $q \in \mathbb{Q}$. If we had $\phi(q) = 0$, then we would have, for every $X \subset \mathbb{Q}$:

\begin{displaymath}
\Phi(X) = \Phi(X) + \phi(q) = \Phi(X \cup \{q\})
\end{displaymath}

But then, for instance, we would have: 

\begin{displaymath}
q = \text{sup}(\{q-1, q\}) = \text{sup}(\{q-1\}) = q-1
\end{displaymath}

This is a contradiction, so $\phi(q) \neq 0$.

Next, note that $\Phi(X)$ must be finite for every upper-bounded $X \subset \mathbb{Q}$ (since sup is undefined for unbounded $X$, we do not consider such sets, and may allow $\Phi$ to diverge). Even if we allowed the domain of $\rho$ to be $\mathbb{R} \cup \{\infty\}$, suppose $\Phi(X) = \infty$ for some upper-bounded set $X$. Then:

\begin{eqnarray*}
\text{sup}(X) & = & \rho(\Phi(X)) \\
& = & \rho(\infty) \\
& = & \rho(\infty + \phi(\text{sup(X) + 1})) \\
& = & \rho(\Phi(X \cup \{sup(X) + 1\})) \\
& = & \text{sup}(X \cup \{sup(X) + 1\}) \\
& = & sup(X) + 1
\end{eqnarray*}

This is a contradiction, so $\Phi(X) < \infty$ for any upper-bounded set $X$.

Now from the above it is immediate that, for any upper-bounded set $X$, only finitely many $x \in X$ can have $\phi(x) > \frac{1}{n}$. Otherwise we can find an infinite upper-bounded set $Y \subset X$ with $\phi(y) > \frac{1}{n}$ for every $y \in Y$, and $\Phi(Y) = \infty$.

Finally, let $q \in \mathbb{Q}$. We have already shown that $\phi(q) \neq 0$, and we will now construct a sequence $q_n$ with:

\begin{enumerate}
    \item $q_n \to q$
    \item $\phi(q_n) \to 0$
\end{enumerate}

If $\phi$ were continuous at $q$, we would have $\phi(q_n) \to \phi(q)$, so the above two points together will give us that $\phi$ is discontinuous at $q$. 

So now, for each $n \in \mathbb{N}$, consider the set $B_n$ of points which lie within $\frac{1}{n}$ of $q$. Since only finitely many points $p \in B_n$ have $\phi(p) > \frac{1}{n}$, and $B_n$ is infinite, there must be a point $q_n \in B_n$ with $\phi(q_n) < \frac{1}{n}$. The sequence of such $q_n$ clearly satisfies both points above, and so $\phi$ is discontinuous everywhere.
\end{proof}

\subsection{\Cref{thm:uncountable_discontinuous}}

\uncdiscontinuous*

\begin{proof}
Define $\Phi : \mathbb{R}^{\mathcal{F}} \to \mathbb{R}$ by $\Phi(X) = \Sigma_{x \in X}\phi(x)$. If we can demonstrate that there exists some $\phi$ such that $\Phi$ is injective, then we can simply choose $\rho = f \circ \Phi^{-1}$ and the result is proved.

Say that a set $X \subset \mathbb{R}$ is \emph{finite-sum-distinct} if, for any finite subsets $A, B \subset X$, $\Sigma_{a \in A}a \neq \Sigma_{b \in B}b$. Now, if we can show that there is a finite-sum-distinct set $D$ with the same cardinality as $\mathbb{R}$ (we denote $|\mathbb{R}|$ by $\mathfrak{c}$), then we can simply choose $\phi$ to be a bijection from $\mathbb{R}$ to $D$. Then, by finite-sum-distinctness, $\Phi$ will be injective, and the result is proved.

Now recall the statement of Zorn's Lemma: suppose $\mathcal{P}$ is a partially ordered set (or \emph{poset}) in which every totally ordered subset has an upper bound. Then $\mathcal{P}$ has a maximal element.

The set of f.s.d. subsets of $\mathbb{R}$ (which we will denote $\mathcal{D}$) forms a poset ordered by inclusion. Supposing that $\mathcal{D}$ satisfies the conditions of Zorn's Lemma, it must have a maximal element, i.e. there is a f.s.d. set $D_\text{max}$ such that any set $E$ with $D_\text{max} \subsetneq E$ is not f.s.d. We claim that $D_\text{max}$ has cardinality $\mathfrak{c}$. 

To see this, let $D$ be a f.s.d. set with infinite cardinality $\kappa < \mathfrak{c}$ (any maximal $D$ clearly cannot be finite). We will show that $D \neq D_\text{max}$. Define the \emph{forbidden elements} with respect to $D$ to be those elements $x$ of $\mathbb{R}$ such that $D \cup \{x\}$ is not f.s.d. We denote this set of forbidden elements $F_D$. Now note that, if $D$ is maximal, then $D \cup F_D = \mathbb{R}$. In particular, this implies that $|F_D|=\mathfrak{c}$. But now consider the elements of $F_D$. By definition of $F_D$, we have that $x \in F_D$ if and only if $\exists c_1, ..., c_m, d_1, ..., d_n \in D$ such that $c_1 + ... + c_m + x = d_1 + ... + d_n$. So we can write $x$ as a sum of finitely many elements of $D$, minus a sum of finitely many other elements of $D$. So there is a surjection from pairs of finite sets of $D$ to elements of $F_D$. i.e.:

\begin{displaymath}
|F_D| \leq |D^{\mathcal{F}} \times D^{\mathcal{F}}|
\end{displaymath}

But since $D$ is infinite:

\begin{displaymath}
|D^{\mathcal{F}} \times D^{\mathcal{F}}| = |D| = \kappa < \mathfrak{c}
\end{displaymath}

So $|F_D| < \mathfrak{c}$, and therefore $|D|$ is not maximal. This demonstrates that $D_\text{max}$ must have cardinality $\mathfrak{c}$.

To complete the proof, it remains to show that $\mathcal{D}$ satisfies the conditions of Zorn's Lemma, i.e. that every totally ordered subset (or \emph{chain}) $\mathcal{C}$ of $\mathcal{D}$ has an upper bound. So consider: 

\begin{displaymath}
C_\text{ub} = \bigcup \mathcal{C} = \bigcup_{C\in\mathcal{C}} C
\end{displaymath}

We claim that $C_\text{ub}$ is an upper bound for $\mathcal{C}$. It is clear that $C \subset C_\text{ub}$ for every $C \in \mathcal{C}$, so it remains to be shown that $C_\text{ub} \in \mathcal{D}$, i.e. that $C_\text{ub}$ is f.s.d.

We proceed by contradiction. Suppose that $C_\text{ub}$ is not f.s.d. Then: 

\begin{equation}
\label{eq:notfsd}
\exists c_1, ..., c_m, d_1, ..., d_n \in C_\text{ub} : \Sigma_i c_i = \Sigma_j d_j
\end{equation}

But now by construction of $C_\text{ub}$ there must be sets $C_1, ..., C_m, D_1, ..., D_m \in \mathcal{C}$ with $c_i \in C_i, d_j \in D_j$. Let $\mathcal{B} = \{C_i\}_{i=1}^m \cup \{D_j\}_{j=1}^n$. $\mathcal{B}$ is totally ordered by inclusion and all sets contained in it are f.s.d., since it is a subset of $\mathcal{C}$. Since $\mathcal{B}$ is finite it has a maximal element $B_\text{max}$. By maximality, we have $c_i, d_j \in B_\text{max}$ for all $c_i, d_j$. But then by \eqref{eq:notfsd}, $B_\text{max}$ is not f.s.d., which is a contradiction. So we have that $C_\text{ub}$ is f.s.d. 

In summary: 

\begin{enumerate}
\item $\mathcal{D}$ satisfies the conditions of Zorn's Lemma.
\item Therefore there exists a maximal f.s.d. set, $D_\text{max}$.
\item We have shown that any such set must have cardinality $\mathfrak{c}$.
\item Given an f.s.d. set $D_\text{max}$ with cardinality $\mathfrak{c}$, we can choose $\phi$ to be a bijection between $\mathbb{R}$ and $D_\text{max}$.
\item Given such a $\phi$, we have that $\Phi(X)=\Sigma_{x \in X} \phi(x)$ is injective on $R^{\mathcal{F}}$.
\item Given injective $\Phi$, choose $\rho = f \circ \Phi^{-1}$.
\item This choice gives us $f(X) = \rho(\Sigma_{x \in X}\phi(x))$ by construction.
\end{enumerate}

This completes the proof.
\end{proof}

\subsection{\Cref{thm:uncountable_only_finite_subsets}}
\uncfinite*

\begin{proof}
Consider $f(X) = \text{sup}(X)$.

As discussed above, a sum over uncountably many elements can converge only if countably many elements are non-zero. But as in the proof of \Cref{thm:q_discontinuous}, $\phi(x) \neq 0$ for any $x$. So it is immediate that sum-decomposition is not possible for functions operating on uncountable subsets of $\mathfrak{X}$.

Even restricting to countable subsets is not enough. As in the proof of \Cref{thm:q_discontinuous}, we must have that for each $n \in \mathbb{N}$, $\phi(x) > \frac{1}{n}$ for only finitely many $x$. But then if this is the case, let $\mathfrak{X}_n$ be the set of all $x \in \mathfrak{X}$ with $\phi(x) > \frac{1}{n}$. Since $\phi(x) \neq 0$, we know that $\mathfrak{X} = \bigcup \mathfrak{X}_n$. But this is a countable union of finite sets, which is impossible because $\mathfrak{X}$ is uncountable.

\end{proof}

\subsection{\Cref{cor:ori_uncountable_theorem}}

\oriunc*

\begin{proof}
The reverse implication is clear. The proof relies on demonstrating that the function $\Phi \to \mathbb{R}^{M+1}$ defined as follows is a homeomorphism onto its image:

\begin{gather*}
\Phi_q(X) = \sum_{m=1}^{M} \phi_q(x_m), \quad q = 0,\dots,M \\
\phi_q(x) = x^q, \quad q = 0,\dots,M
\end{gather*}

Now define $\widetilde{\Phi} \to \mathbb{R}^M$ by:

\begin{gather*}
\widetilde{\Phi}_q(X) = \sum_{m=1}^{M} \widetilde{\phi}_q(x_m), \quad q = 1,\dots,M \\
\widetilde{\phi}_q(x) = x^q, \quad q = 1,\dots,M
\end{gather*}

Note that $\Phi_0(X)=M$ for all $X$, so $\text{Im}(\Phi) = \{M\} \times \text{Im}(\widetilde{\Phi})$. Since $\{M\}$ is a singleton, these two images are homeomorphic, with a homeomorphism given by:

\begin{gather*}
\gamma : \text{Im}(\widetilde{\Phi}) \to \text{Im}(\Phi) \\
\gamma(x_1,\dots,x_M) = (M, x_1, \dots, x_M)
\end{gather*}

Now by definition, $\widetilde{\Phi} = \gamma^{-1} \circ \Phi$. Since this is a composition of homeomorphisms, $\widetilde{\Phi}$ is also a homeomorphism. Therefore $(f \circ \widetilde{\Phi}^{-1}, \widetilde{\phi})$ is a continuous sum-decomposition of $f$ via $\mathbb{R}^M$. 
\end{proof}

\subsection{\Cref{thm:arbitrary_set_sizes}}

\arbitrary*

\begin{proof}
We use the adapted sum-of-power mapping $\widetilde{\Phi}$ from above, denoted in this section by $\Phi$.
\begin{gather*}
\Phi_q(X) = \sum_{m=1}^{M} \phi_q(x_m), \quad q = 1,\dots,M \\
\phi_q(x_m) = (x_m)^q, \quad q = 1,\dots,M
\end{gather*}
which is shown above to be injective. Without loss of generality, let $\mathfrak{X} = [0,1]$ as in \Cref{ori_uncountable_theorem}.

We separate $\Phi_q(X)$ into two terms:
\begin{equation}
    \Phi_q(X) = \sum_{m=1}^{M'} \phi_q(x_m) + \sum_{m=M'+1}^{M} \phi_q(x_m)
\end{equation}
For an input set $X$ with $M'=M-P$ elements and $0 \leq M', P \leq M$, we say that the set contains $M'$ ``actual elements'' as well as $P$ ``empty" elements which are not in fact part of the input set.
Those $P$ ``empty elements'' can be regarded as place fillers when the size of the input set is smaller than $M$, i.e. $M' < M$.

We map those $P$ elements to a constant value $k \notin \mathfrak{X}$, preserving the injectiveness of $\Phi_q(X)$ for input sets $X$ of arbitrary size $M'$:
\begin{equation}
\label{eq:split_e}
    \Phi_q(X) = \sum_{m=1}^{M'} \phi_q(x_m) + \sum_{m=M'+1}^{M} \phi_q(k)
\end{equation}

\Cref{eq:split_e} is no longer strictly speaking a sum-decomposition.
This can be overcome by re-arranging it:
\begin{equation}
\label{eq:split_e_rearranged}
    \begin{split}
    \Phi_q(X) & = \sum_{m=1}^{M'} \phi_q(x_{m}) + \sum_{m=M'+1}^{M} \phi_q(k) \\
           & = \sum_{m=1}^{M'} \phi_q(x_{m}) + \sum_{m=1}^{M} \phi_q(k) - \sum_{m=1}^{M'} \phi_q(k) \\
           & = \sum_{m=1}^{M'} \left[ \phi_q(x_m) - \phi_q(k) \right] + \sum_{m=1}^{M} \phi_q(k)
    \end{split}
\end{equation}

The last term in \Cref{eq:split_e_rearranged} is a constant value which only depends on the choice of $k$ and is independent of $X$ and $M'$.
Hence, we can replace $\phi_q(x)$ by $\widehat{\phi_q}(x)=\phi_q(x)-\phi_q(k)$.
This leads to a new sum-of-power mapping $\widehat{\Phi}_q(X)$ with:
\begin{equation}
\label{eq:new_e}
    \begin{split}
        \widehat{\Phi}_q(X) & = \sum_{m=1}^{M'} \widehat{\phi}_q(x_m)  \\
        & = \Phi_q(X) - M\cdot\phi_q(k)
    \end{split}
\end{equation}

$\widehat{\Phi}$ is injective since $\Phi$ is injective, $k \notin \mathfrak{X}$, and the last term in the above sum is constant. $\widehat{\Phi}$ is also in the form of a sum-decomposition.

For each $m < M$, we can follow the reasoning used in the rest of the proof of \Cref{ori_uncountable_theorem} to note that $\widehat{\Phi}$ is a homeomorphism when restricted to sets of size $m$ -- we denote these restricted functions by $\widehat{\Phi}_m$.
Now each $\widehat{\Phi}_m^{-1}$ is a continuous function into $\mathbb{R}^m$. We can associate with each a continuous function $\widehat{\Phi}_{m,M}^{-1}$ which maps into $\mathbb{R}^M$, with the $M-m$ tailing dimensions filled with the value $k$.

Now the domains of the $\widehat{\Phi}_{m,M}^{-1}$ are compact and disjoint since $k \notin \mathfrak{X}$. We can therefore find a function $\widehat{\Phi}_{C}^{-1}$ which is continuous on $\mathbb{R}^N$ and agrees with each $\widehat{\Phi}_{m,M}^{-1}$ on its domain.

To complete the proof, let $\mathcal{Y}$ be a connected compact set with $k\in\mathcal{Y}, \mathfrak{X}\subset\mathcal{Y}$. Let $\widehat{f}$ be a function on subsets of $\mathcal{Y}$ of size exactly $M$ satisfying:

\begin{gather*}
    \widehat{f}(X) = f(X); \quad X \subset \mathfrak{X} \\
    \widehat{f}(X) = f(X \cap \mathfrak{X}); \quad X \subset \mathfrak{X} \cup \{k\}
\end{gather*}

We can choose $\widehat{f}$ to be continuous under the notion of continuity in \Cref{sec:cont_set_fun_remark}.
Then $(\widehat{f}\circ\widehat{\Phi}_{C}^{-1}, \widehat{\phi})$ is a continuous sum-decomposition of $f$.

\end{proof}

\subsection{Max-Decomposition}
\label{app:max-decomp}

Analogously to sum-decomposition, we define the notion of \emph{max-decomposition}. A function $f$ is max-decomposable if there are functions $\rho$ and $\phi$ such that:

\begin{displaymath}
f(\mathbf{x}) = \rho \bigl( \text{max}_i ( \phi(x_i) ) \bigr).
\end{displaymath}

where the max is taken over each dimension independently in the latent space. Our definitions of decomposability via $Z$ and continuous decomposability also extend to the notion of max-decomposition.

We now state and prove a theorem which is closely related to \Cref{thm:max_not_decomposable}, but which establishes limitations on max-decomposition, rather than sum-decomposition.

\begin{theorem}
\label{thm:sum_not_decomposable}
Let $M > N \in \mathbb{N}$. Then there exist permutation invariant continuous functions $f : \mathbb{R}^M \to \mathbb{R}$ which are not max-decomposable via $\mathbb{R}^N$.
\end{theorem}

Note that this theorem rules out any max-decomposition, whether continuous or discontinuous. We specifically demonstrate that summation is not max-decomposable -- as with \Cref{thm:max_not_decomposable}, this theorem applies to ordinary well-behaved functions.

\begin{proof}
Consider $f(\mathbf{x}) = \sum_{i=1}^M x_m$. Let $\phi : \mathbb{R} \to \mathbb{R}^N$, and let $\mathbf{x} \in \mathbb{R}^M$ such that $x_i \neq x_j$ when $i \neq j$.

For $n=1,\dots,N$, let $\mu(n) \in \{1,\dots,M\}$ such that:

\begin{displaymath}
\text{max}_i ( \phi(x_i)_n ) = \phi(x_{\mu(n)})_n
\end{displaymath}

That is, $\phi(x_{\mu(q)})$ attains the maximal value in the $q$-th dimension of the latent space among all $\phi(x_i)$. Now since $N < M$, there is some $m \in \{1,\dots,M\}$ such that $\mu(n) \neq m$ for any $n \in \{1,\dots,N\}$. So now consider $\widetilde{\mathbf{x}}$ defined by:

\begin{gather}
\widetilde{x}_i = x_i ; i \neq m \\
\widetilde{x}_m = x_{\mu(1)}
\end{gather}

Then:

\begin{displaymath}
\text{max}_i ( \phi(x_i) ) = \text{max}_i ( \phi(\widetilde{x}_i) )
\end{displaymath}

But since we chose $\mathbf{x}$ such that all $x_i$ were distinct, we have $\sum_{i=1}^M x_i \neq \sum_{i=1}^M \widetilde{x}_i$ by the definition of $\widetilde{\mathbf{x}}$. This shows that $\phi$ cannot form part of a max-decomposition for $f$. But $\phi$ was arbitrary, so no max-decomposition exists.

\end{proof}

\section{A Continuous Function on $\mathbb{Q}$}
\label{sec:function_on_q}

This section defines and analyses the function $\Psi$ shown in \Cref{fig:q_continuous}, which is continuous on $\mathbb{Q}$ but not on $\mathbb{R}$. $\Psi$ is defined as the pointwise limit of a sequence of functions $\Psi_n$, illustrated in \Cref{fig:q_continuous_progression}. We proceed as follows:

\begin{enumerate}
    \item Define a sequence of functions $\widetilde{\Psi}_n$ on $[0,1]$.
    \item Show that the pointwise limit $\widetilde{\Psi}$ is continuous except at points of the form $k\cdot 2^{-m}$ for some integers $k$ and $m$, i.e. except at the \emph{dyadic rationals}.
    \item Define the function $\Psi$ on $[0, A]$ by $\Psi(x) = \widetilde{\Psi}(\frac{x}{A})$.
    \item Note that $\Psi$ is continuous except at points of the form $A\cdot k\cdot 2^{-m}$ for some integers $k$ and $m$.
    \item Choose $A$ to be irrational, so that all points of discontinuity are also irrational, to obtain a function which is continuous on $\mathbb{Q}$. (In all figures, we have chosen $A=\text{log}(4)$).
\end{enumerate}

\begin{figure}[ht!]
\input{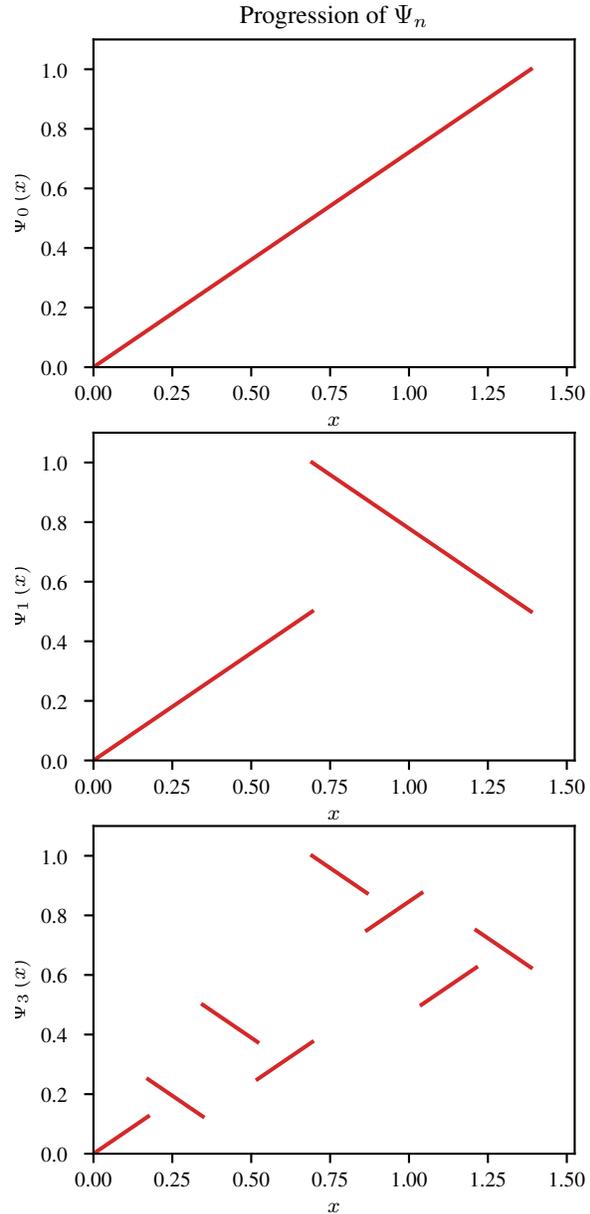}
\vspace{-10mm}
\caption{Several iterations of $\Psi_n$}
\label{fig:q_continuous_progression}
\end{figure}

Informally, we set $\widetilde{\Psi_0}(x)=x$, and at iteration $n$, we split the unit interval into $2^n$ even subintervals. In every even-numbered subinterval, we reflect the function horizontally around the midpoint of the subinterval. We may write this formally as follows.

Let $x \in [0,1], n \in \mathbb{N}$. Let:
\begin{displaymath}
a_n(x) = \frac{\ceil{x\cdot 2^n}+\frac{1}{2}}{2^n}
\end{displaymath}
That is, $a_n(x)$ is the midpoint of the unique half-open interval containing $x$: 
\begin{displaymath}
\Big(k\cdot 2^{-n}, (k+1)\cdot 2^{-n}\Big] ; \quad k \in \mathbb{N}
\end{displaymath}
Write $b_n(x)$ for the $n$-th digit in the binary expansion of $x$, and write $c_n(x)$ for the number of $b_m(x), m \leq n$ with $b_m(x) = 1$.

Importantly, $b_n(x)$ is ambiguous if $x$ is a dyadic rational, since in this case $x$ has both a terminating and a non-terminating expansion. For consistency with our choice of the upward-closed interval for the definition of $a_n(x)$, we choose the non-terminating expansion in this case.

Then:

\begin{gather*}
\widetilde{\Psi}_n(x)=x + \sum_{i=1}^n (-1)^{c_i(x)}\cdot b_i(x) \cdot 2(x - a_i(x)) \\
\widetilde{\Psi}(x)=x + \sum_{i=1}^\infty (-1)^{c_i(x)}\cdot b_i(x) \cdot 2(x - a_i(x))
\end{gather*}

First, it is clear that the series for $\widetilde{\Psi}(x)$ converges absolutely at every $x$, since $|2(x - a_i(x))| \leq 2^{-i}$. So this function is well defined. Also note that:
\begin{equation}
\label{eq:psi_cont}
|\widetilde{\Psi}_n(x) - \widetilde{\Psi}(x)| \leq 2^{-n}
\end{equation}
Note further that $\widetilde{\Psi}_n$ is continuous except at points of the form $k\cdot 2^{-n}$, since $a_m$, $b_m$ and $c_m$ are continuous at these points for all $m \leq n$.

Now consider a point $x_*$ which is not a dyadic rational. We wish to show that $\widetilde{\Psi}$ is continuous at $x_*$. So let $\epsilon > 0$. Choose $n$ so that $2^{-n} < \frac{\epsilon}{3}$. Since $x_*$ is not a dyadic rational, $\widetilde{\Psi}_n$ is continuous at $x_*$, i.e. there is some $\delta > 0$ such that $|x_* - x| < \delta \implies |\widetilde{\Psi}_n(x_*) - \widetilde{\Psi}_n(x)| < \frac{\epsilon}{3}$. But now, by \Cref{eq:psi_cont}:

\begin{gather*}
|\widetilde{\Psi}(x_*) - \widetilde{\Psi}_n(x_*)| < 2^{-n} < \frac{\epsilon}{3} \\
|\widetilde{\Psi}(x) - \widetilde{\Psi}_n(x)| < 2^{-n} < \frac{\epsilon}{3}
\end{gather*}

And so, whenever $|x_* - x| < \delta$, we have:

\begin{eqnarray*}
|\widetilde{\Psi}(x_*) - \widetilde{\Psi}(x)| & \leq & |\widetilde{\Psi}(x_*) - \widetilde{\Psi}_n(x_*)| \\
&& + |\widetilde{\Psi}_n(x_*) - \widetilde{\Psi}_n(x)| \\
&& + |\widetilde{\Psi}_n(x) - \widetilde{\Psi}(x)| \\
& < & \frac{\epsilon}{3} + \frac{\epsilon}{3} + \frac{\epsilon}{3} \\
|\widetilde{\Psi}(x_*) - \widetilde{\Psi}(x)| & < & \epsilon
\end{eqnarray*}

Thus, $\widetilde{\Psi}$ is continuous at $x_*$.

As noted above, we may now set $\Psi(x) = \widetilde{\Psi}(\frac{x}{A})$ to obtain a function which is continuous at all rational points.

\section{Implementation Details for Illustrative Example}
\label{sec:app_exp}

The network setup follows \Cref{eq:main} with 3 fully connected layers before the summation (acting on each input independently) and 2 fully connected layers after. Each fully connected layer has 1000 hidden units and is followed by a ReLU non-linearity. However, the third hidden layer, which creates the latent space in which the summation is executed, has a variable dimension $N$. $N$ is varied for each experiment in order to examine the influence of the latent dimension on the performance.

Training was conducted using the ADAM optimizer \citep{kingma2014adam} with an initial learning rate of $0.001$ and an exponential decay after each batch of $0.99$. Training was ended after convergence at 500 batches with a batch size of $32$. Samples were continuously drawn from the respective distributions. Therefore, there is no notion of training vs. test data or epoch sizes. The results $r^t$, measured as RMSE, were smoothed using exponential smoothing with $\alpha = 0.95$:
\begin{equation}
    r_{\text{smooth}}^t = (1-\alpha) \cdot r^t  + \alpha \cdot r^{t-1}
\end{equation}

The last, smoothed RMSE is extracted from each experiment, averaged over 500 different runs with different seeds and plotted in \Cref{fig:smooth_logxy}. The confidence intervals are calculated assuming a Gaussian distribution. The critical points are extracted by taking the smallest latent dimension which produces an RMSE of less than $10\%$ above the global minimum for this set size.

\emph{Out of distribution samples: }
We tested the performance of a trained model (500 inputs, 100 latent dimensions) on out-of-distribution samples to see to what extent the model exploits the statistical properties of the training examples. Below is a list with examples:
\begin{itemize}
    \item Input: $[1.0, 1.0, \dots, 1.0]$, output: $0.933$, true label: $1.0$
    \item Input: 200 times $1.0$ and 300 times $0.0$, output: $0.413$, true label: $0.0$
    \item Input: $[0.002, 0.004, 0.006, ..., 1.0]$, output: $0.5006$, true label: $0.5000$
\end{itemize}
The distributions the samples were drawn from during test time were the uniform distribution, a Gaussian and a Gamma distribution. Each sample consisting of 500 values was randomly drawn from one of these distributions. Hence, the first two are very unlikely samples from the provided distributions. The poor performance of the model on these two examples can therefore be taken as an indication that the model does utilize information about the underlying distributions when estimating the median. The third sample is much closer to a realistic sample from, e.g., the uniform distribution (in our case between $0.0$ and $1.0$), which makes it unsurprising that the model performs much better on this task. It is worth noting that the notion of 'likely' examples of uniform distributions is of course an intuitive one. Given that no value is repeated, every specific set of numbers is of course equally likely as long as all numbers lie within the interval of the uniform distribution.

\end{document}